\documentclass{article}

\usepackage{microtype}
\usepackage{graphicx}
\usepackage{subcaption}
\usepackage{booktabs} 
\usepackage{tikz}
\usetikzlibrary{arrows.meta}
\usepackage{algpseudocode}
\usepackage{xcolor} 

\usepackage{hyperref}

\usepackage{algorithm}
\usepackage{algorithmic}
\usepackage{xcolor}

\usepackage[preprint]{icml2026}

\usepackage{amsmath}
\usepackage{amssymb}
\usepackage{mathtools}
\usepackage{amsthm}

\usepackage[capitalize,noabbrev]{cleveref}

\theoremstyle{plain}
\newtheorem{theorem}{Theorem}[section]

\theoremstyle{definition}

\newtheorem{assumption}[theorem]{Assumption}
\theoremstyle{remark}

\usepackage[textsize=tiny]{todonotes}

\icmltitlerunning{Dependency-Guided Parallel Decoding in Discrete Diffusion Language Models}

\begin{document}

\onecolumn
  \icmltitle{Dependency-Guided Parallel Decoding in Discrete Diffusion Language Models}


  \icmlsetsymbol{equal}{*}

  \begin{icmlauthorlist}
    \icmlauthor{Liran Ringel}{yyy}
    \icmlauthor{Ameen Ali}{xxx}
    \icmlauthor{Yaniv Romano}{yyy,zzz}
  \end{icmlauthorlist}

  \icmlaffiliation{yyy}{Department of Computer Science, Technion – Israel Institute of Technology}
  \icmlaffiliation{xxx}{Blavatnik School of Computer Science and AI, Tel Aviv, Israel}
  \icmlaffiliation{zzz}{Department of Electrical and Computer Engineering, Technion – Israel Institute of Technology}
  \icmlcorrespondingauthor{Liran Ringel}{liranringel@cs.technion.ac.il}

  \icmlkeywords{Diffusion language models, parallel sampling, non-autoregressive generation, text generation}

  \vskip 0.3in



\printAffiliationsAndNotice{}  

\begin{abstract}
Discrete diffusion language models (dLLMs) accelerate text generation by unmasking multiple tokens in parallel. However, parallel decoding introduces a distributional mismatch: it approximates the joint conditional using a fully factorized product of per-token marginals, which degrades output quality when selected tokens are strongly dependent.

We propose DEMASK (DEpendency-guided unMASKing), a lightweight dependency predictor that attaches to the final hidden states of a dLLM. In a single forward pass, it estimates pairwise conditional influences between masked positions. Using these predictions, a greedy selection algorithm identifies positions with bounded cumulative dependency for simultaneous unmasking. Under a sub-additivity assumption, we prove this bounds the total variation distance between our parallel sampling and the model's joint. Empirically, DEMASK achieves 1.7--2.2$\times$ speedup on Dream-7B while matching or improving accuracy compared to confidence-based and KL-based baselines.

Code: \href{https://github.com/liranringel/demask}{GitHub} \,\textbar\, Model: \href{https://huggingface.co/liranringel/Dream-v0-Instruct-7B-Demask}{Hugging Face}
\end{abstract}

\section{Introduction}
\label{sec:intro}

\begin{figure}[h]
    \centering
    \includegraphics[width=0.6\columnwidth]{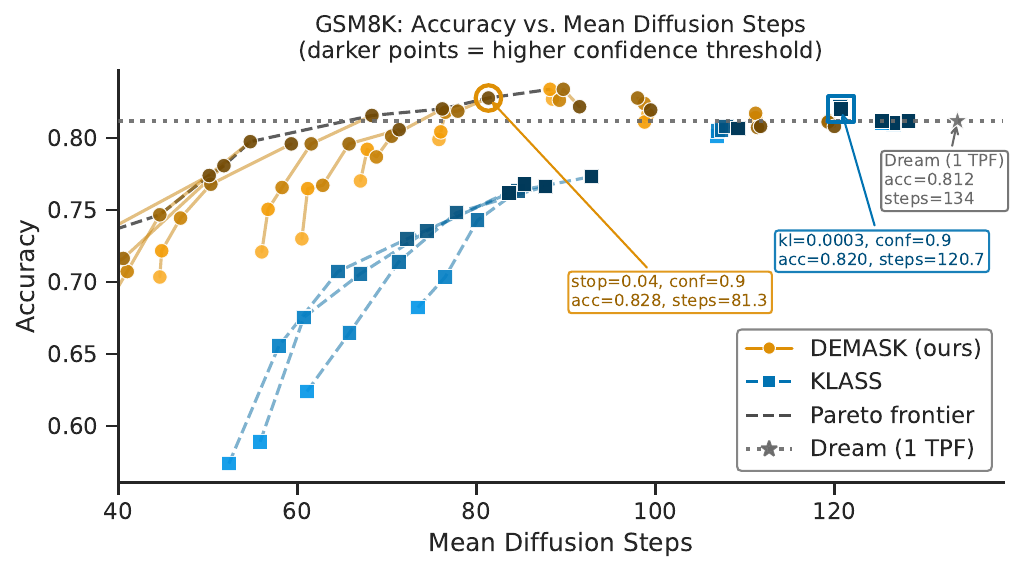}
    \caption{Accuracy vs.\ mean diffusion steps (forward passes) on GSM8K for DEMASK and KLASS. Each point represents a hyperparameter configuration; darker points indicate higher confidence thresholds. The Pareto frontier (dashed) shows DEMASK dominates across the efficiency-accuracy trade-off. Dream (1 TPF) denotes 1 token per forward pass with entropy selection.}
    \label{fig:grid_search}
\end{figure}

Discrete diffusion language models (dLLMs)~\citep{shi2024simplified,zheng2025masked,ye2025dream,nie2025large} have emerged as an alternative to autoregressive generation, iteratively denoising masked sequences while leveraging bidirectional context. A key advantage is the ability to unmask multiple tokens per diffusion step, accelerating inference compared to strictly left-to-right decoding---though potentially at the cost of output quality.

Parallel decoding introduces a distributional mismatch: while the model's conditional distribution factorizes via the chain rule, parallel sampling draws from a product of per-token marginals over a selected subset $S$ of positions, i.e., $Q_\theta(Y_S) = \prod_{i \in S} P_\theta(Y_i \mid \text{context})$ instead of the joint $P_\theta(Y_S \mid \text{context})$. This discrepancy degrades output quality when selected tokens exhibit strong dependencies. Consider ``$1 + \square = \square$'': each masked position's marginal distribution independently favors common digits like ``2'' or ``3'', but sampling both in parallel risks producing inconsistent results such as ``$1 + 2 = 5$'' since neither position conditions on the other. Unmasking one position before the other allows the second sample to condition on the first, restoring consistency.

Our key insight is that if we can identify which masked positions are approximately conditionally independent, we can unmask them simultaneously without distributional mismatch. We propose to identify, at each diffusion step, a subset of masked positions whose cumulative pairwise dependency remains bounded. To this end, we introduce a lightweight dependency predictor that estimates pairwise conditional influences from the hidden states of a single forward pass (Figure~\ref{fig:predictor_arch}), enabling a greedy algorithm to select approximately independent subsets whose cumulative dependency remains below a threshold $\tau$.

We provide theoretical justification: under a sub-additivity assumption (validated empirically), our greedy selection guarantees $\text{TV}(P_\theta, Q_\theta) \leq \tau$ between the model's joint $P_\theta$ and the factorized approximation $Q_\theta$. This bound directly limits how often parallel sampling produces different outputs than sequential sampling.

Figure~\ref{fig:grid_search} previews a key empirical result: DEMASK consistently dominates the Pareto frontier of accuracy vs.\ inference speed, achieving higher accuracy at any given step budget compared to KLASS~\citep{kim2025klass}, a recent KL-based baseline. 

In sum, our contributions are: (1) A single-forward-pass predictor for pairwise dependencies, enabling dependency-bounded parallel sampling in dLLMs. (2) Theoretical bounds on distributional deviation under a sub-additivity assumption. (3) Empirical validation showing 1.7--2.2$\times$ speedup with comparable or improved accuracy over confidence-based and KL-based baselines. (4) Training and evaluation code and trained predictor weights are publicly available on \href{https://github.com/liranringel/demask}{GitHub}.

\section{Related Work}
\label{sec:related_work}

\subsection{Autoregressive Language Models}

Autoregressive language models \cite{radford2018improving} generate text by predicting the next token given all previous tokens, factorizing the joint distribution via the chain rule: $P(x) = \prod_t P(x_t \mid x_{<t})$. The Transformer architecture~\citep{vaswani2017attention} enabled scaling to billions of parameters, with GPT-3~\citep{brown2020language} demonstrating emergent few-shot capabilities. Recent autoregressive models~\citep{dubey2024llama,liu2024deepseek,zeng2024chatglm,yang2025qwen3,team2025kimi} achieve remarkable performance across reasoning, coding, and multilingual tasks. However, autoregressive decoding is inherently sequential: generating $n$ tokens requires $n$ forward passes, limiting throughput. This motivates parallel decoding paradigms such as discrete diffusion.

\subsection{Discrete Diffusion Language Models}

Masked diffusion models (MDMs), often called diffusion large language models (dLLMs),~\citep{austin2021structured,lou2024discrete,shi2024simplified,sahoo2024simple} generate text by iteratively predicting masked tokens. The model uses bidirectional attention (no causal mask) and is trained to maximize a lower bound on the data log-likelihood (see \Cref{sec:preliminaries}).

Recent work has scaled MDMs to billions of parameters. LLaDA~\citep{nie2025large} trains an 8B-parameter model from scratch, achieving performance competitive with LLaMA3. Dream~\citep{ye2025dream} initializes from a pretrained autoregressive model and uses context-adaptive noise scheduling to reach 7B parameters. LLaDA 2.0~\citep{bie2025llada2} extends this conversion approach to 100B parameters. These large-scale models expose the marginal-joint gap that motivates our approach.

\subsection{Parallel Decoding Strategies}

Most existing methods only implicitly mitigate the marginal-joint gap discussed in \Cref{sec:intro} through confidence-based selection. Dream and LLaDA select tokens based on entropy or top-$k$ probability. LLaDA remasks a scheduled fraction of low-confidence predictions at each step. KLASS~\citep{kim2025klass} uses KL divergence between consecutive timesteps, unmasking tokens whose predicted distributions remain stable. dParallel~\citep{chen2025dparallel} uses certainty-forcing distillation to reduce the number of denoising steps. \citet{patel2025improved} propose position contrastive guidance, biasing toward left-to-right decoding. In contrast with our approach, these methods treat tokens independently and do not explicitly model inter-token dependencies.

\subsection{Dependency-Aware Decoding}

Recent work explicitly addresses the marginal-joint gap by modeling token dependencies. PUNT~\citep{azangulov2025parallel} employs a divide-and-conquer strategy to identify an approximately independent subset using $\mathcal{O}(\log|M|)$ forward passes per step, where $|M|$ is the number of masked tokens. Discrete Copula Diffusion~\citep{liu2025discrete} captures dependencies via a separate autoregressive model. APD~\citep{israel2025accelerating} combines dLLM marginals with joint probabilities from a small auxiliary autoregressive model via a multiplicative mixture. \citet{bansal2025enabling} train a lightweight sampler layer to approximate joint sampling from a frozen dLLM. \citet{jazbec2025learning} learn unmasking policies via reinforcement learning, though without explicitly modeling token dependencies.

Unlike these approaches, our method selects the least dependent token positions by predicting their dependencies in a single forward pass without requiring auxiliary models. Additionally, our selection rule has solid theoretical foundations.

\section{Preliminaries}
\label{sec:preliminaries}

\paragraph{Discrete Diffusion Language Models.}
Discrete diffusion language models generate text through iterative denoising of masked sequences~\citep{nie2025large,ye2025dream}. Let $Y = (Y_1, \ldots, Y_N) \in \mathcal{V}^N$ denote a response sequence over vocabulary $\mathcal{V}$, conditioned on a prompt $X$. During training, time $t \in [0,1]$ indexes the corruption level: at each $t$, positions are independently replaced with a special \texttt{[MASK]} token with probability $\alpha_t$, where $\alpha_0 = 0$ (clean) and $\alpha_1 = 1$ (fully masked). The model is trained to maximize a lower bound on the data log-likelihood; in practice, this reduces to minimizing a time-weighted cross-entropy over masked positions~\citep{shi2024simplified,sahoo2024simple,ye2025dream,bie2025llada2}:
\begin{equation}\label{eq:dllm_loss}
    \mathcal{L}(\theta) = \mathbb{E}\left[ w(t) \sum_{i \in M_t} -\log P_\theta(Y_i \mid X, Y_{U_t}) \right]
\end{equation}
where $M_t$ denotes the masked positions at time $t$, $Y_{U_t}$ the visible tokens, and $w(t)$ a time-dependent weight. At inference time, generation begins from a fully masked sequence and proceeds through iterative refinement: each step involves a forward pass to obtain token distributions, selection of positions to unmask, and sampling new tokens at those positions. The selection strategy---which positions to unmask per step---governs the trade-off between parallelism and output quality.

\paragraph{Total Variation Distance.} We define here this distance as it will be used later to bound the approximation error between the factorized and joint distributions.
For distributions $P$ and $Q$ over a discrete sample space $\mathcal{X}$, the total variation (TV) distance is defined as
\begin{equation}
    \textup{TV}(P, Q) = \frac{1}{2} \sum_{x \in \mathcal{X}} |P(x) - Q(x)|.
\end{equation}
TV is a metric on probability distributions, bounded in $[0,1]$, and satisfies the triangle inequality.

\section{Methodology}
\label{sec:method}

\begin{figure*}[t]
    \centering
    \includegraphics[width=\textwidth]{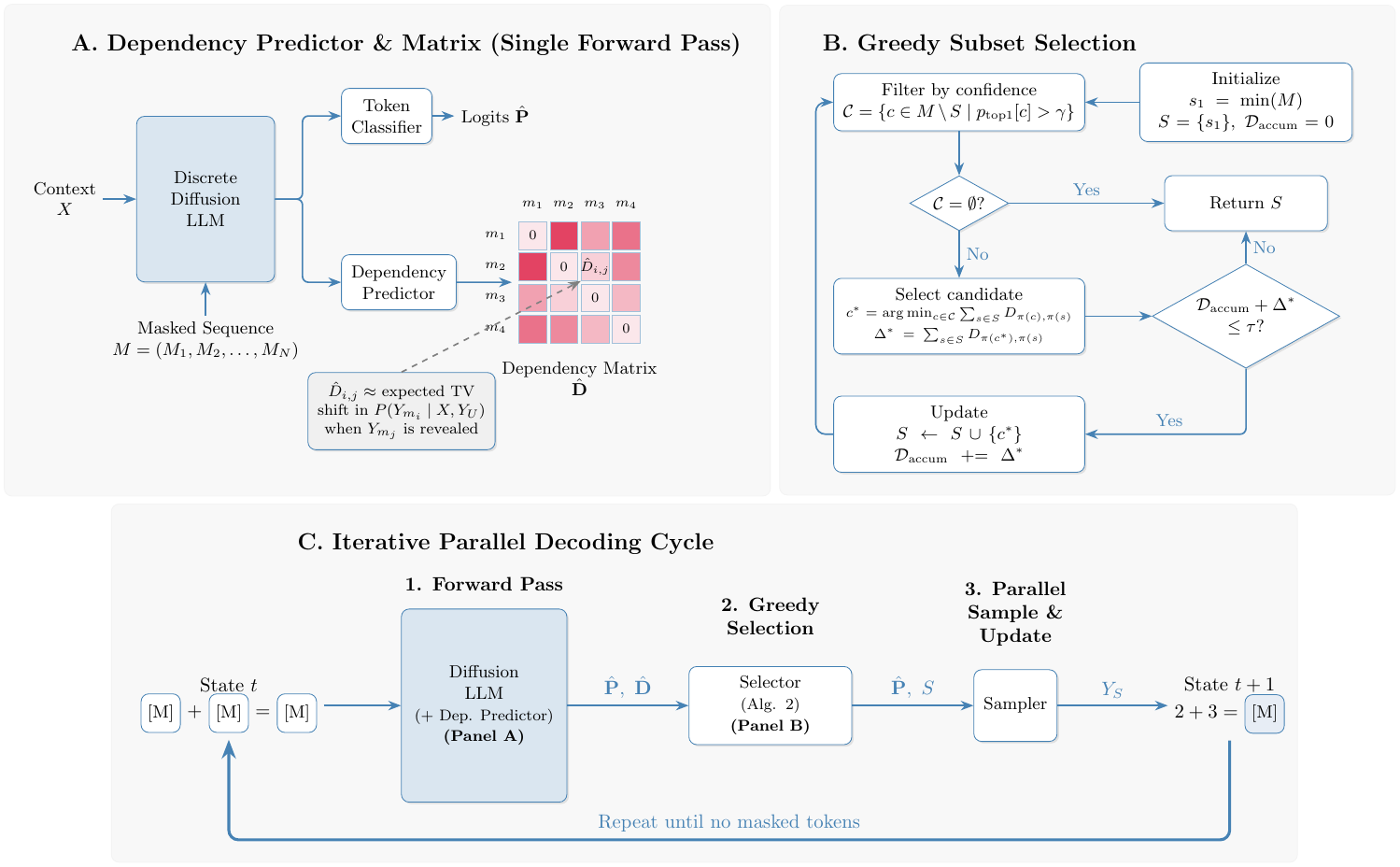}
    \caption{Overview of DEMASK. \textbf{(A)}~A lightweight dependency predictor attaches to the dLLM backbone and estimates pairwise dependencies $\hat{\mathbf{D}}$ from hidden states in a single forward pass. \textbf{(B)}~Greedy subset selection identifies positions with bounded cumulative dependency for parallel unmasking. \textbf{(C)}~The iterative decoding cycle: each step performs a forward pass, selects positions, and samples them in parallel until all tokens are unmasked.}
    \label{fig:diagram}
\end{figure*}

To accelerate generation, we propose an iterative unmasking procedure (Algorithm~\ref{alg:generation}) where each diffusion step identifies a subset of masked positions $S=\{s_1, \ldots, s_{|S|}\}$ to unmask simultaneously. This subset is constructed using a greedy selection strategy (Algorithm~\ref{alg:selection}), selecting positions one by one such that their values can be sampled with minimal deviation from the model's joint distribution. Crucially, $S$ is constructed using the output of a single forward pass. Figure~\ref{fig:diagram} illustrates our approach.

We first describe the greedy selection procedure (\S\ref{sec:greedy}), then provide theoretical guarantees (\S\ref{sec:theory}), and finally detail the dependency predictor that enables efficient selection (\S\ref{sec:predictor}).

\subsection{Greedy Position Selection}
\label{sec:greedy}

Let $X$ denote the tokens of the input prompt. We partition the indices of the current response sequence into two sets: $U$ (unmasked) and $M$ (masked). Let $Y_U = (Y_{u_1}, \ldots, Y_{u_{|U|}}) \in \mathcal{V}^{|U|}$ denote the tokens at positions $U$ which were determined in previous diffusion steps, where $\mathcal{V}$ is the vocabulary. Analogously, let $Y_S = (Y_{s_1}, \ldots, Y_{s_{|S|}}) \in \mathcal{V}^{|S|}$ denote the random variables corresponding to the tokens at the selected positions $S \subseteq M$, and let $Y_M \in \mathcal{V}^{|M|}$ denote the variables corresponding to the set of masked positions $M=\{m_1, \ldots, m_{|M|}\}$. Our goal is to select the largest subset $S$ such that, conditioned on $X$ and $Y_U$, the dependence among the variables in $Y_S$ remains below a specified threshold.

In diffusion language models, parallel sampling of multiple tokens amounts to sampling from their marginal distributions. While the model's joint distribution follows the chain rule decomposition $P_\theta(Y_S \mid X, Y_U) = \prod_{j=1}^{|S|} P_\theta(Y_{s_j} \mid Y_{s_{<j}}, X, Y_U)$, where $s_{<j} = \{s_1, \dots, s_{j-1}\}$, parallel sampling effectively draws from a fully factorized approximation $Q_\theta(Y_S \mid X, Y_U)$, defined as:
\begin{equation}\label{eq:Q}
    Q_\theta(Y_S \mid X, Y_U) = \prod_{i \in S} P_\theta(Y_i \mid X, Y_U)
\end{equation}
where each token is sampled independently given the context.

We construct the set $S$ using a greedy approximation strategy, described in Algorithm~\ref{alg:selection}. Let $M = \{m_1, \dots, m_{|M|}\}$ be the ordered set of masked positions. Let $\mathbf{D} \in [0,1]^{|M| \times |M|}$ be the pairwise dependency matrix, where the entry $D_{i,j}$ represents the expected change in the conditional distribution of the variable at position $m_i$ when the variable at position $m_j$ is revealed:
\begin{equation}\label{eq:D_mat}
\begin{split}
    D_{i,j} = \mathbb{E}_{\textcolor{blue}{Y_{m_j}} \sim P_\theta(\cdot \mid X, Y_U)} \Big[ \textup{TV}\Big( & P_\theta(Y_{m_i} \mid X, Y_U), \\
    & P_\theta(Y_{m_i} \mid X, Y_U, \textcolor{blue}{Y_{m_j}})\Big) \Big].
\end{split}
\end{equation}
Here, $\textup{TV}(P, Q) = \frac{1}{2} \sum_{y \in \mathcal{V}} |P(y) - Q(y)|$. Intuitively, $D_{i,j} \approx 0$ implies that $Y_{m_i}$ is nearly conditionally independent of $Y_{m_j}$ given the prompt $X$ and history $Y_U$. Note that the matrix $\mathbf{D}$ is generally asymmetric. At inference time, we use $\hat{\mathbf{D}}$, the output of a learned dependency predictor trained to approximate these values. The details on how we learn to predict $\hat{\mathbf{D}}$ are provided later in Section~\ref{sec:predictor}; first, we explain how the knowledge of the pairwise dependencies $\mathbf{D}$ can be used to select $S$.

The algorithm initializes the set $S$ by selecting the left-most masked position $s_1 = \min(M)$. This left-to-right bias reflects the sequential structure of natural language reasoning: prior work has shown that biasing toward earlier positions improves generation quality~\citep{patel2025improved}, and that token-order selection outperforms confidence-based selection when generating one token per step~\citep{israel2025accelerating}. Notably, our method can also be implemented without this left-to-right bias; it is merely a design choice that we found to work well.

Subsequently, the algorithm proceeds iteratively. At each step, we first identify a set of candidate positions $\mathcal{C} \subseteq M \setminus S$ where the model's confidence—measured by the top-1 probability—exceeds a threshold $\gamma$. This filtering follows the common practice in previous work, showing that confidence-based selection improves performance \cite{bie2025llada2,kim2025klass}. In our context, this confidence-based selection is necessary because approximate independence does not imply low uncertainty: tokens may be approximately independent yet individually uncertain when their values depend on context not yet revealed. Without confidence filtering, the algorithm may commit to highly uncertain tokens early.

From these high-confidence candidates $\mathcal{C}$, we select the position $c^*$ that minimizes the aggregate pairwise dependency on the currently selected positions $S$:
\begin{equation}
    c^* = \underset{c \in \mathcal{C}}{\arg\min} \sum_{s \in S} D_{\pi(c), \pi(s)},
\end{equation}
where $\pi: M \to \{1, \dots, |M|\}$ maps a global position in $M$ to its corresponding index in the matrix $\mathbf{D}$: $\pi(m_i) = i$. The algorithm continues to add positions to $S$ until either no candidates remain (i.e., $\mathcal{C} = \emptyset$) or the cumulative dependency exceeds $\tau$. Formally, if $\mathcal{D}_{\text{accum}}$ is the current accumulated dependency, we stop if $\mathcal{D}_{\text{accum}} + \sum_{s \in S} D_{\pi(c^*), \pi(s)} > \tau$. This strategy maximizes the degree of parallelism for accelerated generation while strictly limiting the deviation from the model's joint distribution to a user-specified bound, as formally proven in Section~\ref{sec:theory}.

\begin{algorithm}[t]
\caption{Iterative Diffusion Generation}
\label{alg:generation}
\begin{algorithmic}[1]
\Require Prompt $X$, Sequence Length $N$, Thresholds $\tau, \gamma$
\State $U \gets \emptyset, \quad M \gets \{1, \dots, N\}$

\While{$M \neq \emptyset$}
    \vspace{0.2em}
    \State \textbf{\textit{Phase 1: Forward Pass}}
        \State $\hat{\mathbf{P}}, \hat{\mathbf{D}} \gets \text{Model}(X, Y_U, M)$
        \Statex \hspace{\algorithmicindent} {\small \color{gray} $\triangleright$ Dims: $\hat{\mathbf{P}} \in [0,1]^{|M| \times |\mathcal{V}|}, \enspace \hat{\mathbf{D}} \in [0,1]^{|M| \times |M|}$}
        
        \State $p_{\text{top1}} \gets \max(\hat{\mathbf{P}}, \text{dim}=\text{vocab})$
        \Statex \hspace{\algorithmicindent} {\small \color{gray} $\triangleright$ Dims: $p_{\text{top1}} \in [0,1]^{|M|}$}

    \Statex
    \State \textbf{\textit{Phase 2: Greedy Subset Selection}} (Algorithm~\ref{alg:selection})
        \State $S \gets \text{GreedySubsetSelection}(\hat{\mathbf{D}}, M, p_{\text{top1}}, \gamma, \tau)$

    \Statex
    \State \textbf{\textit{Phase 3: Parallel Sample \& Update}}
        \State $Y_S \sim \text{Sample}(\hat{\mathbf{P}}, S)$ \hspace{0.1cm}// Sample at positions $S$
        \State $U \gets U \cup S$
        \State $M \gets M \setminus S$
        \State $Y_U \gets \text{Update}(Y_U, Y_S)$
\EndWhile
\State \Return $Y_U$
\end{algorithmic}
\end{algorithm}

\begin{algorithm}[t]
\caption{Greedy Subset Selection}
\label{alg:selection}
\begin{algorithmic}[1]
\Require Dependency matrix $\mathbf{D}$, Mask positions $M$, Top-1 probs $p_{\text{top1}}$, Confidence threshold $\gamma$, Dependency Bound $\tau$
\State $S \gets \emptyset, \quad \mathcal{D}_{\text{accum}} \gets 0$
\State $s_1 \gets \min(M)$ \hspace{0.1cm}// Start with left-most
\State $S \gets S \cup \{s_1\}$

\While{\textbf{true}}
    \State \textbf{\textit{Step 1: Filter Candidates by Confidence}}
    \State $\mathcal{C} \gets \{c \in M \setminus S \mid p_{\text{top1}}[c] > \gamma\}$
    \If{$\mathcal{C} = \emptyset$} \textbf{break} \EndIf

    \Statex
    \State \textbf{\textit{Step 2: Select Candidate Minimizing Dependency}}
    \State $c^* \gets \arg\min_{c \in \mathcal{C}} \sum_{s \in S} D_{\pi(c), \pi(s)}$
    \State $\Delta^* \gets \sum_{s \in S} D_{\pi(c^*), \pi(s)}$

    \Statex
    \State \textbf{\textit{Step 3: Check Cumulative Dependency Constraint}}
    \If{$\mathcal{D}_{\text{accum}} + \Delta^* > \tau$} \textbf{break} \EndIf
    \State $S \gets S \cup \{c^*\}$
    \State $\mathcal{D}_{\text{accum}} \gets \mathcal{D}_{\text{accum}} + \Delta^*$
\EndWhile
\State \Return $S$
\end{algorithmic}
\end{algorithm}

\subsection{Theoretical Analysis}
\label{sec:theory}

We now provide a theoretical grounding for the proposed greedy selection strategy, assuming that the true pairwise dependency matrix $\mathbf{D}$ is known. Our objective is to ensure that the joint distribution of the variables at the selected positions $S$ does not deviate significantly from the fully factorized approximation.

We measure the approximation error between the model's joint distribution $P_\theta(Y_S \mid X, Y_U)$ and the factorized approximation $Q_\theta(Y_S \mid X, Y_U)$ using the TV distance:
\begin{equation}
    \mathcal{E}(S) = \textup{TV}\Big(P_\theta(Y_S \mid X, Y_U), Q_\theta(Y_S \mid X, Y_U)\Big).
\end{equation}
This metric admits a rigorous operational interpretation via the concept of maximal coupling. A coupling is any joint distribution over pairs $(A,B)$ whose marginals are $P_A$ and $P_B$, respectively. It is a standard result that the minimum mismatch probability $\mathbb{P}(A \neq B)$ among all couplings is exactly $\textup{TV}(P_A, P_B)$, and this minimum is attained by the maximal coupling \citep{levin2017markov}. In our case, let $A \sim P_\theta(Y_S \mid X, Y_U)$ follow the model's joint and $B \sim Q_\theta(Y_S \mid X, Y_U)$ follow the factorized distribution. Then, the minimum mismatch probability is $\mathbb{P}(A \neq B) = \mathcal{E}(S)$. Consequently, if $\mathcal{E}(S) \leq \tau$, then under the optimal coupling $(A,B)$, we have $\mathbb{P}(A = B) \geq 1 - \tau$ \citep[Proposition 4.7]{levin2017markov}.

To bound this error using only pairwise statistics, we rely on the following assumption regarding the sub-additivity of dependencies.

\begin{assumption}[Sub-Additivity]
\label{ass:subadditivity}

We assume the dependency of a variable on its history is sub-additive:
\begin{align}
  \begin{split}
      & \mathbb{E}_{Y_{s_{<t}}} \left[ \textup{TV}\Big(P_\theta(Y_{s_t} \mid X, Y_U), P_\theta(Y_{s_t} \mid X, Y_U, Y_{s_{<t}})\Big) \right] \\
      & \quad \le \sum_{j=1}^{t-1} D_{\pi(s_t), \pi(s_j)},
  \end{split}
\end{align}
where $Y_{s_{<t}} = (Y_{s_1}, \dots, Y_{s_{t-1}})$ denotes the random variables at previously selected positions.
\end{assumption}

That is, we assume that the joint influence of the history $Y_{s_{<t}}$ on $Y_{s_t}$ does not exceed the sum of individual pairwise influences $D_{\pi(s_t), \pi(s_j)}$ from \eqref{eq:D_mat}. In Section~\ref{sec:synergy_validation}, we provide empirical evidence on the Tulu 3 SFT Mixture dataset demonstrating that this assumption holds in practice. With this assumption in place, we prove that the deviation between the joint and factorized distributions is bounded by $\tau$ for the selected set obtained by our greedy algorithm. See Appendix~\ref{sec:proof_thm_correctness} for the complete proof.

\begin{theorem}[Correctness of Algorithm~\ref{alg:selection}] \label{thm:correctness}
Suppose that Assumption~\ref{ass:subadditivity} holds and that Algorithm~\ref{alg:selection} is implemented with the true matrix $\mathbf{D}$ and error threshold $\tau$. The selected subset $S$ returned by Algorithm~\ref{alg:selection} satisfies $\textup{TV}(P_\theta(Y_S \mid X, Y_U), Q_\theta(Y_S \mid X, Y_U)) \le \tau$.
\end{theorem}
\noindent \textbf{Remark.} When using an estimated $\hat{\mathbf{D}}$, the bound holds approximately, with the additional error depending on the quality of the predictor.

\subsection{Learning the Dependency Predictor}
\label{sec:predictor}

The greedy selection strategy described above requires access to the pairwise dependency matrix $\mathbf{D}$ defined in~\eqref{eq:D_mat}. Computing this matrix exactly requires $O(|M|)$ additional forward passes---one for each masked position---which would negate the efficiency gains of parallel decoding. We address this by training a lightweight predictor that estimates $\hat{\mathbf{D}}$ from the hidden states of a single forward pass.

\paragraph{Architecture.}
Our predictor draws inspiration from scaled dot-product attention~\citep{vaswani2017attention}. Let $d$ denote the hidden dimension of the backbone. Given final-layer hidden states $\mathbf{H} \in \mathbb{R}^{|M| \times d}$ at the masked positions (batch dimension omitted for clarity), we compute query and key representations $\mathbf{Q} = \mathbf{H} \mathbf{W}_Q$ and $\mathbf{K} = \mathbf{H} \mathbf{W}_K$, where $\mathbf{W}_Q, \mathbf{W}_K \in \mathbb{R}^{d \times d}$ are learned projections. The predicted dependency is:
\begin{equation}
    \hat{D}_{i,j} = \sigma\left( \frac{\mathbf{q}_i \cdot \mathbf{k}_j}{\sqrt{d}} \right)
\end{equation}
where $\mathbf{q}_i, \mathbf{k}_j$ are the $i$-th and $j$-th rows of $\mathbf{Q}, \mathbf{K}$ respectively, and $\sigma(\cdot)$ is sigmoid. The sigmoid ensures outputs lie in $[0, 1]$, matching the TV range. We zero out the diagonal since self-dependencies are irrelevant. Note that $\hat{\mathbf{D}}$ is generally asymmetric ($\hat{D}_{i,j} \neq \hat{D}_{j,i}$), consistent with~\eqref{eq:D_mat}. We observed empirically that this Q/K factorization yields better optimization performance. Importantly, at inference time, we merge the projections into $\mathbf{W} = \mathbf{W}_Q \mathbf{W}_K^\top$ and compute $\hat{D}_{i,j} = \sigma(\mathbf{h}_i \mathbf{W} \mathbf{h}_j^\top / \sqrt{d})$ directly to reduce the two matrix-matrix multiplications (required to compute $\mathbf{Q}$ and $\mathbf{K}$) into one. Figure~\ref{fig:predictor_arch} illustrates this architecture.

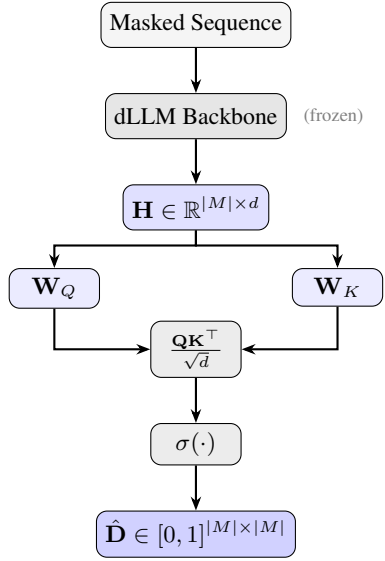
\begin{figure}[t]
\centering
\begin{tikzpicture}[
    node distance=0.6cm and 0.8cm,
    box/.style={rectangle, draw, rounded corners, minimum width=1.8cm, minimum height=0.6cm, align=center, font=\small},
    smallbox/.style={rectangle, draw, rounded corners, minimum width=1.2cm, minimum height=0.5cm, align=center, font=\footnotesize},
    arrow/.style={-{Stealth[scale=0.8]}, thick},
    annot/.style={font=\scriptsize, text=gray}
]

\node[box, fill=gray!8] (input) {Masked Sequence};

\node[box, fill=gray!20, below=of input] (backbone) {dLLM Backbone};
\node[annot, right=0.1cm of backbone] {(frozen)};

\node[box, fill=blue!12, below=of backbone] (hidden) {$\mathbf{H} \in \mathbb{R}^{|M| \times d}$};

\node[smallbox, fill=blue!8, below left=0.5cm and 0.3cm of hidden] (wq) {$\mathbf{W}_Q$};
\node[smallbox, fill=blue!8, below right=0.5cm and 0.3cm of hidden] (wk) {$\mathbf{W}_K$};

\node[smallbox, fill=gray!15, below=1.2cm of hidden] (dot) {$\frac{\mathbf{Q}\mathbf{K}^\top}{\sqrt{d}}$};

\node[smallbox, fill=gray!15, below=of dot] (sigmoid) {$\sigma(\cdot)$};

\node[box, fill=blue!18, below=of sigmoid] (output) {$\hat{\mathbf{D}} \in [0,1]^{|M| \times |M|}$};

\draw[arrow] (input) -- (backbone);
\draw[arrow] (backbone) -- (hidden);
\draw[arrow] (hidden.south) -- ++(0,-0.2) -| (wq.north);
\draw[arrow] (hidden.south) -- ++(0,-0.2) -| (wk.north);
\draw[arrow] (wq.south) |- (dot.west);
\draw[arrow] (wk.south) |- (dot.east);
\draw[arrow] (dot) -- (sigmoid);
\draw[arrow] (sigmoid) -- (output);

\end{tikzpicture}
\caption{Dependency predictor architecture. Hidden states $\mathbf{H}$ from the frozen backbone are projected via learned $\mathbf{W}_Q, \mathbf{W}_K$, then combined via scaled dot-product and sigmoid to predict the pairwise dependency matrix $\hat{\mathbf{D}}$.}
\label{fig:predictor_arch}
\end{figure}

\paragraph{Two-Phase Training Pipeline.}
Training the dependency predictor proceeds in two phases: (i)~\emph{TV cache generation}, where we precompute ground-truth dependency matrices using the full backbone, and (ii)~\emph{predictor training}, where we train only the lightweight predictor on the cached targets. This decoupling offers key advantages: we can run multiple epochs over the cached data without recomputing expensive TV matrices; and we can iterate on predictor architectures and hyperparameters without re-running the costly data generation.

\paragraph{Phase 1: TV Cache Generation.}
We generate training data by explicitly computing pairwise TV distances. For each training sample, we apply random masking with ratio $t \sim \text{Uniform}(0, 1)$. Recall that~\eqref{eq:D_mat} defines $D_{i,j}$ as an expectation over $Y_{m_j}$. We define $D_{i,j}(y_{m_j})$ as the TV distance for a specific realization $y_{m_j} \sim P_\theta(\cdot \mid X, Y_U)$: 
\begin{equation*}\label{eq:D_sample}
  D_{i,j}(y_{m_j}) = \textup{TV}\Big(P_\theta(Y_{m_i} \mid X, Y_U), P_\theta(Y_{m_i} \mid X, Y_U, y_{m_j})\Big)
\end{equation*}
so that $D_{i,j} = \mathbb{E}_{Y_{m_j}\sim P_\theta(\cdot \mid X, Y_U)}[D_{i,j}(Y_{m_j})]$. For each masked position $m_j$, we sample a realization $y_{m_j}$ and store $D_{i,j}(y_{m_j})$ for all $i$. This requires $|M|+1$ forward passes per sample: one with all positions masked to obtain $P_\theta(Y_{m_i} \mid X, Y_U)$, and one for each revealed token $y_{m_j}$ to obtain $P_\theta(Y_{m_i} \mid X, Y_U, y_{m_j})$. To increase diversity, we generate 5 samples per response using different random masks and $y_{m_j}$ draws. Our training set comprises 100K responses from the Tulu-3 SFT mixture~\citep{lambert2025tulu}, yielding approximately 500K training examples. Generating this cache requires 3.5 hours on 8$\times$H200 GPUs.

\paragraph{Phase 2: Predictor Training.}
With the TV cache in hand, we train only the dependency predictor while keeping the backbone frozen. Our objective is to minimize the expected squared error:
\begin{equation}\label{eq:objective}
\mathbb{E}\left[ \left( \hat{D}_{i,j} - D_{i,j}(Y_{m_j}) \right)^2 \right],
\end{equation}
where the expectation is over the training distribution $(X, Y_U, M)$, uniform sampling of pairs $(i,j)$ with $i \neq j$, and $Y_{m_j} \sim P_\theta(\cdot \mid X, Y_U)$. We minimize an unbiased estimate of this objective by averaging over all off-diagonal pairs in a batch:
\begin{equation}\label{eq:loss}
    \mathcal{L} = \frac{1}{\sum_{b=1}^{B} (|M^{(b)}|^2 - |M^{(b)}|)} \sum_{b=1}^{B} \sum_{i \neq j} \left( \hat{D}_{i,j}^{(b)} - D_{i,j}^{(b)}(y_{m_j}^{(b)}) \right)^2
\end{equation}
where $B$ is the batch size, and $y_{m_j}^{(b)}$ are $B$ realizations of $Y_{m_j}$ from Phase 1. Diagonal entries are excluded since self-dependencies are irrelevant. Crucially, the MSE-optimal predictor outputs the conditional mean, so the trained model learns to approximate $\mathbb{E}_{Y_{m_j} \sim P_\theta(\cdot \mid X,Y_U)}[D_{i,j}(Y_{m_j})] = D_{i,j}$, in line with~\eqref{eq:D_mat}. This is because $y_{m_j}$ in \eqref{eq:loss} are sampled from $P_\theta(\cdot \mid X,Y_U)$. The backbone remains frozen in bfloat16, while the predictor is trained in float32 for numerical stability. We use AdamW with learning rate $10^{-5}$, weight decay $0.01$, and cosine schedule with 5\% warmup. We train for 5 epochs (approx.\ 95 minutes on 8$\times$H200 GPUs) and select the checkpoint with the lowest validation loss. The predictor adds only $2d^2 \approx 26$M parameters (for $d = 3584$), negligible compared to the 7B backbone.

\section{Experiments}

\begin{table*}[t]
\centering
\caption{Comparison of DEMASK and KLASS on Dream-7B. Each cell shows accuracy (\%) with speedup in parentheses, relative to Entropy (1 tok). \textbf{Bold}: best average accuracy (Avg excludes GSM8K).}
\label{tab:dream-klass}
\small
\setlength{\tabcolsep}{4pt}
\begin{tabular}{@{}lccccc@{}}
\toprule
Method & MMLU-Pro & GSM8K & HumanEval & MBPP & Avg \\
\midrule
Entropy (1 tok) & 43.9$\pm$0.4 (1.0$\times$) & 81.2$\pm$1.1 (1.0$\times$) & 53.7$\pm$3.9 (1.0$\times$) & 58.0$\pm$2.2 (1.0$\times$) & 51.8 (1.0$\times$) \\
Entropy (2 tok) & 39.7$\pm$0.4 (2.0$\times$) & 76.7$\pm$1.2 (2.2$\times$) & 31.1$\pm$3.6 (1.1$\times$) & 43.8$\pm$2.2 (1.3$\times$) & 38.2 (1.5$\times$) \\
Top-1 (1 tok) & 45.5$\pm$0.4 (1.0$\times$) & 81.2$\pm$1.1 (1.0$\times$) & 53.7$\pm$3.9 (1.0$\times$) & 58.0$\pm$2.2 (1.0$\times$) & 52.4 (1.0$\times$) \\
Top-1 (2 tok) & 43.5$\pm$0.4 (2.0$\times$) & 76.7$\pm$1.2 (2.1$\times$) & 31.1$\pm$3.6 (1.1$\times$) & 43.8$\pm$2.2 (1.4$\times$) & 39.5 (1.5$\times$) \\
Token Order (1 tok) & 45.7$\pm$0.4 (1.3$\times$) & 80.7$\pm$1.1 (1.1$\times$) & 55.5$\pm$3.9 (1.3$\times$) & 57.8$\pm$2.2 (1.5$\times$) & 53.0 (1.3$\times$) \\
Token Order (2 tok) & 43.6$\pm$0.4 (2.8$\times$) & 73.8$\pm$1.2 (2.1$\times$) & 24.4$\pm$3.4 (1.1$\times$) & 40.8$\pm$2.2 (1.5$\times$) & 36.2 (1.8$\times$) \\
KLASS & 45.7$\pm$0.4 (1.3$\times$) & 82.0$\pm$1.1 (1.1$\times$) & 55.5$\pm$3.9 (1.1$\times$) & 58.2$\pm$2.2 (1.6$\times$) & 53.1 (1.3$\times$) \\
DEMASK (ours) & 47.5$\pm$0.4 (1.7$\times$) & 82.8$\pm$1.0 (1.8$\times$) & 54.9$\pm$3.9 (1.7$\times$) & 57.6$\pm$2.2 (2.2$\times$) & \textbf{53.3 (1.9$\times$)} \\
\bottomrule
\end{tabular}
\end{table*}

\subsection{Experimental Setup}

We evaluate DEMASK on Dream-7B~\citep{ye2025dream}, a 7B-parameter dLLM trained on instruction-following data. We compare against several baselines: (i)~\textbf{Entropy}: select tokens with lowest entropy (the baseline used in the original Dream paper); (ii)~\textbf{Top-1}: select tokens with highest top-1 probability; (iii)~\textbf{KLASS}~\citep{kim2025klass}: select stable tokens that pass both a KL divergence threshold (between consecutive timesteps) and a top-1 confidence threshold; and (iv)~\textbf{Token Order}: unmask tokens left-to-right. For (i), (ii), and (iv), we evaluate with 1 and 2 tokens per step. All methods use temperature 0.1 and top-$p$ sampling with $p=0.9$, following the original Dream evaluation setup.

We report results on four benchmarks: MMLU-Pro~\citep{wang2024mmlu} (reasoning), GSM8K~\citep{cobbe2021training} (math reasoning), HumanEval~\citep{chen2021evaluating} and MBPP~\citep{austin2021program} (code generation). We use the lm-evaluation-harness library~\cite{eval-harness} to run the evaluations. All experiments run on 8$\times$H200 GPUs with data parallelism (batch size 1 per GPU). We report wall-clock time relative to Entropy with 1 token per forward pass.

\paragraph{Avoid sampling after EOS tokens.} We apply an optimization to all methods: once an end-of-sequence (EOS) token is sampled, all subsequent positions are immediately set to EOS. The model typically fills these positions with EOS tokens on its own, but burns inference time doing so. This optimization avoids that wasted computation and ensures fair comparison across methods. As a result, metrics may differ from those reported in prior work. We discuss the impact of this optimization in Appendix~\ref{app:eos_ablation}.

\subsection{Hyperparameter Selection}

DEMASK has two hyperparameters: the dependency threshold $\tau$ (lower values enforce stricter independence) and the confidence threshold $\gamma$ (higher values require more confident predictions). KLASS similarly has two hyperparameters: a KL threshold (unmask only if the predicted distribution is close to that of the previous timesteps) and a confidence threshold. In all experiments, we use a history length of 2 (see \citet{kim2025klass} for details). We conduct a grid search on GSM8K for both methods; see Appendix~\ref{app:hyperparams} for the full search grids.

Figure~\ref{fig:grid_search} shows accuracy vs.\ mean diffusion steps for DEMASK and KLASS. DEMASK dominates the Pareto frontier across the entire trade-off curve, achieving higher accuracy at comparable step counts. We select $\tau=0.04$ and $\gamma=0.9$ for DEMASK, and KL threshold $0.0003$ with a confidence threshold $0.9$ for KLASS, as favorable trade-off points. We use these hyperparameters for both methods in all subsequent experiments.
To ensure a fair comparison, we re-optimized the hyperparameters for KLASS rather than using the values reported in the original paper, accounting for differences in the temperature parameter and evaluation pipeline.

\subsection{Main Empirical Results}

Table~\ref{tab:dream-klass} presents our main empirical results. DEMASK achieves the highest average accuracy (53.3\%) while providing 1.9$\times$ average speedup over the Entropy baseline. Notably, DEMASK improves accuracy on MMLU-Pro (+3.6\% over Entropy), suggesting that dependency-aware selection can improve output quality in addition to accelerating generation. We exclude GSM8K from the average accuracy and speedup calculations, as it was used for hyperparameter optimization.

\paragraph{Baselines degrade with parallelism.} All baselines (Entropy, Top-1, and Token Order) suffer significant accuracy drops when unmasking 2 tokens per step instead of 1. This degradation validates our central hypothesis: unmasking multiple tokens without accounting for their dependencies introduces errors that compound through generation.

\subsection{Empirical Validation of Sub-Additivity}
\label{sec:synergy_validation}

\begin{figure}[h]
    \centering
    \includegraphics[width=0.5\columnwidth]{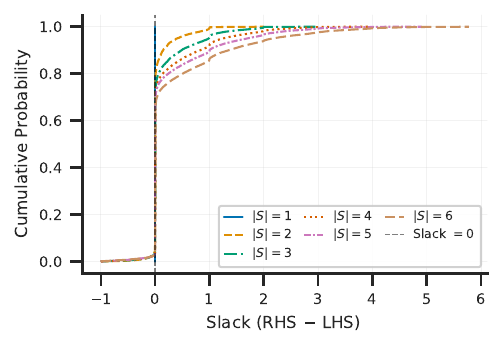}
    \caption{Empirical CDF of the slack $\mathrm{RHS}_i - \mathrm{LHS}_i$ stratified by subset size $|S|$, evaluated on Tulu 3 SFT Mixture with Dream-7B. The $|S|=1$ curve (vertical line at zero) is trivially satisfied. For $|S| \geq 2$, positive slack indicates the sub-additivity bound holds.}
    \label{fig:synergy_cdf}
\end{figure}

We empirically validate Assumption~\ref{ass:subadditivity} (Sub-Additivity) on the Tulu 3 SFT Mixture dataset~\citep{lambert2025tulu} using Dream-7B~\citep{ye2025dream}. This assumption underpins the theoretical guarantees of our greedy selection algorithm (Theorem~\ref{thm:correctness}).

\paragraph{Experimental Setup.}
We evaluate on 5,000 prompt--response pairs from the dataset, drawing 5 independent random masks per response. For each masked sample, we apply the same random masking procedure used during training: sample a mask ratio $t \sim \mathrm{Uniform}(0,1)$ and mask the corresponding fraction of response positions to obtain a masked set $M$. We then uniformly sample a subset size $|S| \sim \mathrm{Uniform}\{1,\dots,6\}$ and draw $S \subset M$ uniformly at random among subsets of that size. For each target position $i \in M \setminus S$, we sample token values $y_S$ from the model's conditional distribution and compute:
\begin{equation}\label{eq:synergy_lhs_rhs}
\begin{aligned}
    \mathrm{LHS}_i &= \mathrm{TV}\bigl(P_\theta(Y_i \mid X, Y_U),\; P_\theta(Y_i \mid X, Y_U, Y_S = y_S)\bigr), \\
    \mathrm{RHS}_i &= \sum_{j \in S} \mathrm{TV}\bigl(P_\theta(Y_i \mid X, Y_U),\; P_\theta(Y_i \mid X, Y_U, Y_j = y_j)\bigr),
\end{aligned}
\end{equation}
where $P_\theta(Y_i \mid X, Y_U)$ denotes the distribution with all positions in $M$ masked. Computing these quantities requires $|S|+2$ forward passes per sample: one with all tokens masked, one for each single-token reveal in $S$, and one for the joint reveal. The sub-additivity assumption holds when $\mathrm{RHS}_i \geq \mathrm{LHS}_i$; we refer to the difference $\mathrm{RHS}_i - \mathrm{LHS}_i$ as the \emph{slack}.

\paragraph{Results.}
Figure~\ref{fig:synergy_cdf} shows the empirical CDF of the slack stratified by subset size $|S|$. Across all subset sizes, the distribution is concentrated to the right of zero, with violation rates below 6\%. Notably, the mean slack increases with $|S|$, indicating that the bound becomes increasingly slack for larger subsets. These results confirm that violations are rare, supporting Assumption~\ref{ass:subadditivity}'s use in Section~\ref{sec:method}.

\section{Conclusion}

We presented DEMASK, a dependency-guided approach to parallel decoding in discrete diffusion language models. Rather than treating token selection as independent decisions, our method learns to predict which positions exhibit weak dependence and can therefore be unmasked simultaneously with minimal distributional error.

On Dream-7B, DEMASK delivers speedups of up to 2.2$\times$ without sacrificing output quality; the accuracy even improves on several benchmarks. Our approach achieves performance gains across reasoning, math, and code generation tasks, while requiring only a small auxiliary module (${\sim}26$M parameters).

\paragraph{Future Directions.}
Several directions merit further investigation. First, the predictor architecture is agnostic to the backbone, requiring only hidden states at masked positions, which should enable extension to other dLLM architectures such as LLaDA and LLaDA 2.0. Second, improved predictor architectures---such as multi-head variants or deeper networks---may capture finer-grained dependencies, though with an accuracy-cost trade-off. Third, combining dependency-guided selection with remasking strategies could leverage dependency structure to decide which tokens to remask. Fourth, tighter theoretical bounds might be achievable by relaxing the sub-additivity assumption. Finally, adaptive thresholds that vary $\tau$ and $\gamma$ based on generation progress or task characteristics could further improve the accuracy-efficiency trade-off.

\section{Limitations}

\paragraph{Sub-additivity assumption.}
Our theoretical guarantee (Theorem~\ref{thm:correctness}) relies on Assumption~\ref{ass:subadditivity}. As shown in Section~\ref{sec:synergy_validation}, this assumption holds in the vast majority of cases. When violations occur, the TV bound may not hold exactly; however, we observe no corresponding degradation in downstream task performance, suggesting the method is robust to such violations.

\paragraph{Theory-practice gap in dependency estimation.}
Theorem~\ref{thm:correctness} assumes access to the true dependency matrix $\mathbf{D}$, whereas our implementation uses a learned approximation $\hat{\mathbf{D}}$. Prediction errors propagate to the TV bound, though we have not formally characterized this relationship. Empirically, the predictor appears sufficiently accurate for the downstream tasks we evaluated.

\paragraph{Single backbone architecture.}
We evaluate DEMASK exclusively on Dream-7B. Applying the method to other dLLM architectures requires retraining the dependency predictor on data generated from those backbones, which we have not yet performed.

\section*{Impact Statement}

This paper presents work whose goal is to advance the field of machine learning. There are many potential societal consequences of our work, none of which we feel must be specifically highlighted here.
\section*{Acknowledgments}
L. R. and Y. R. were supported by the European Union (ERC, SafetyBounds, 101163414). Views and opinions
expressed are however those of the authors only and do not necessarily reflect those of the European Union or
the European Research Council Executive Agency. Neither the European Union nor the granting authority can be
held responsible for them. This research was also partially supported by the Israel Science Foundation (ISF grant
729/21). Y. R. acknowledges additional support from the Career Advancement Fellowship at the Technion. The contribution of the first author is part of a
PhD thesis research conducted at the Technion.

\bibliography{example_paper}
\bibliographystyle{icml2026}

\newpage
\appendix
\onecolumn

\section{Proof of \Cref{thm:correctness}}
\label{sec:proof_thm_correctness}

\begin{proof}
We use a telescoping sum argument. Let $k = |S|$ denote the number of selected positions. Define a sequence of intermediate distributions $Q^{(t)}$ over the random variables $Y_S$ for $t=1, \dots, k$. Let $Q^{(t)}$ be the distribution that preserves the model's joint dependency for the first $t$ variables and assumes independence for the remainder. For a specific realization $y_S = (y_{s_1}, \dots, y_{s_k})$, this is defined as:
$$ Q^{(t)}(y_S \mid X, Y_U) = P_\theta(y_{s_1}, \dots, y_{s_t} \mid X, Y_U) \prod_{j=t+1}^k P_\theta(y_{s_j} \mid X, Y_U). $$
Observe that $Q^{(k)}$ corresponds to the model's joint distribution $P_\theta(Y_S \mid X, Y_U)$. Furthermore, $Q^{(1)}$ corresponds to the factorized approximation $Q_\theta(Y_S \mid X, Y_U)$ defined in \eqref{eq:Q}, since the first token $s_1$ has no history within $S$. By the Triangle Inequality:
\begin{equation}\label{eq:triangle}
\textup{TV}(P_\theta(Y_S \mid X, Y_U), Q_\theta(Y_S \mid X, Y_U)) \le \sum_{t=2}^k \textup{TV}(Q^{(t)}, Q^{(t-1)}).
\end{equation}

We now analyze the term $\textup{TV}(Q^{(t)}, Q^{(t-1)})$ for $t \ge 2$. The distributions $Q^{(t)}$ and $Q^{(t-1)}$ differ only in the generation of the variable $y_{s_t}$: in $Q^{(t)}$, $y_{s_t}$ is drawn from the conditional distribution $P_\theta(y_{s_t} \mid X, Y_U, y_{s_{<t}})$, whereas in $Q^{(t-1)}$, it is drawn from the marginal distribution $P_\theta(y_{s_t} \mid X, Y_U)$ independent of the history $y_{s_{<t}} = (y_{s_1}, \dots, y_{s_{t-1}})$.

To facilitate the calculation, we partition the full realization vector $y_S$ into three parts: the history $y_{s_{<t}}$, the current token $y_{s_t}$, and the future tokens $y_{s_{>t}} = (y_{s_{t+1}}, \ldots, y_{s_k})$. We express the TV distance by summing over all possible realizations $y_S \in \mathcal{V}^k$:
\begin{align}
    \textup{TV}(Q^{(t)}, Q^{(t-1)})
    &= \frac{1}{2} \sum_{y_S \in \mathcal{V}^k} \left| Q^{(t)}(y_S \mid X, Y_U) - Q^{(t-1)}(y_S \mid X, Y_U) \right| \nonumber \\
    &= \frac{1}{2} \sum_{y_S \in \mathcal{V}^k} \left| \textcolor{blue}{P_\theta(y_{s_{<t}} \mid X, Y_U)} P_\theta(y_{s_t} \mid X, Y_U, y_{s_{<t}}) \textcolor{violet}{\prod_{j=t+1}^k P_\theta(y_{s_j} \mid X, Y_U)} \right. \nonumber \\
    &\quad \quad \quad \quad \quad \left. - \textcolor{blue}{P_\theta(y_{s_{<t}} \mid X, Y_U)} P_\theta(y_{s_t} \mid X, Y_U) \textcolor{violet}{\prod_{j=t+1}^k P_\theta(y_{s_j} \mid X, Y_U)} \right| \label{eq:tv_expansion}
\end{align}
We observe that the probability of the history \textcolor{blue}{$P_\theta(y_{s_{<t}} \mid X, Y_U)$} and the probability of the future tokens \textcolor{violet}{$\prod_{j=t+1}^k P_\theta(y_{s_j} \mid X, Y_U)$} are common factors. We factor them out, organizing the terms to match the logical order of dependency:
\begin{align}
    \eqref{eq:tv_expansion} &= \frac{1}{2} \sum_{y_S \in \mathcal{V}^k} \textcolor{blue}{P_\theta(y_{s_{<t}} \mid X, Y_U)} \cdot \left| P_\theta(y_{s_t} \mid X, Y_U, y_{s_{<t}}) - P_\theta(y_{s_t} \mid X, Y_U) \right| \cdot \textcolor{violet}{\left( \prod_{j=t+1}^k P_\theta(y_{s_j} \mid X, Y_U) \right)}. \label{eq:tv_factorized}
\end{align}
We now split the summation over $y_S$ into three nested sums corresponding to the history $y_{s_{<t}}$, the current token $y_{s_t}$, and the future tokens $y_{s_{>t}}$:
$$ \sum_{y_S \in \mathcal{V}^k} (\dots) = \sum_{y_{s_{<t}} \in \mathcal{V}^{t-1}} \sum_{y_{s_t} \in \mathcal{V}} \sum_{y_{s_{>t}} \in \mathcal{V}^{k-t}} (\dots). $$
Substituting this into \eqref{eq:tv_factorized}, the terms depending on $y_{s_{>t}}$ (the purple block) are clearly separated at the end. Since the vocabulary $\mathcal{V}$ is finite, we can rigorously marginalize out the future variables. Specifically, because the future variables are sampled independently, the summation over $y_{s_{>t}}$ and the product over $j$ commute. This allows us to push the sum inside the product:
\begin{align}
    \eqref{eq:tv_factorized} &= \frac{1}{2} \sum_{y_{s_{<t}}} \textcolor{blue}{P_\theta(y_{s_{<t}} \mid X, Y_U)} \sum_{y_{s_t}} \left| P_\theta(y_{s_t} \mid X, Y_U, y_{s_{<t}}) - P_\theta(y_{s_t} \mid X, Y_U) \right| \cdot \left( \sum_{y_{s_{>t}} \in \mathcal{V}^{k-t}} \textcolor{violet}{\prod_{j=t+1}^k P_\theta(y_{s_j} \mid X, Y_U)} \right) \label{eq:before_swap} \\
    &= \frac{1}{2} \sum_{y_{s_{<t}}} \textcolor{blue}{P_\theta(y_{s_{<t}} \mid X, Y_U)} \sum_{y_{s_t}} \left| P_\theta(y_{s_t} \mid X, Y_U, y_{s_{<t}}) - P_\theta(y_{s_t} \mid X, Y_U) \right| \cdot \underbrace{\left( \prod_{j=t+1}^k \left[ \sum_{y_{s_j} \in \mathcal{V}} \textcolor{violet}{P_\theta(y_{s_j} \mid X, Y_U)} \right] \right)}_{=1 \times \dots \times 1 = 1} \label{eq:marginalization} \\
    &= \mathbb{E}_{Y_{s_{<t}}} \left[ \textup{TV}\Big(P_\theta(Y_{s_t} \mid X, Y_U), P_\theta(Y_{s_t} \mid X, Y_U, Y_{s_{<t}})\Big) \right]. \label{eq:expected_tv}
\end{align}
The right hand side term in Eq.~\eqref{eq:marginalization} equals 1 because each inner sum is over the complete support $\mathcal{V}$ of a valid probability distribution. Finally, we apply Assumption~\ref{ass:subadditivity} (Sub-Additivity) to \eqref{eq:expected_tv}, which bounds this expected deviation by the sum of pairwise dependencies:
$$ \mathbb{E}_{Y_{s_{<t}}} \left[ \textup{TV}\Big(P_\theta(Y_{s_t} \mid X, Y_U), P_\theta(Y_{s_t} \mid X, Y_U, Y_{s_{<t}})\Big) \right] \le \sum_{j=1}^{t-1} D_{\pi(s_t), \pi(s_j)}. $$
Substituting this back into~\eqref{eq:triangle}:
\begin{equation}
    \textup{TV}(P_\theta(Y_S \mid X, Y_U), Q_\theta(Y_S \mid X, Y_U)) \le \sum_{t=2}^k \underbrace{\left( \sum_{j=1}^{t-1} D_{\pi(s_t), \pi(s_j)} \right)}_{\text{Cost of adding } s_t}.
\end{equation}
The inner term corresponds exactly to the cost added by Algorithm~\ref{alg:selection} when selecting the $t$-th token. The outer sum represents the total accumulated cost. Since the algorithm initializes the cost to zero for the first token (as there are no prior dependencies in $S$) and stops before the total exceeds $\tau$, the bound holds.
\end{proof}

\section{Ablation: Avoid sampling after EOS tokens}
\label{app:eos_ablation}

In discrete diffusion language models, sampling an end-of-sequence (EOS) token at any position signals that generation should terminate. However, the original Dream implementation continues unmasking subsequent positions even after EOS is sampled. While these positions do not affect the final output, they still require diffusion steps.

We apply a natural optimization: once an EOS token is sampled, all subsequent positions are immediately set to EOS. This avoids wasted computation and reflects the intended semantics of EOS. We apply this optimization uniformly to all methods (baselines, KLASS, and DEMASK) for fair comparison.

Table~\ref{tab:eos-ablation} shows results \emph{without} this optimization. Without EOS optimization, KLASS runs much faster (e.g., 6.9 minutes on MBPP vs 19.3 minutes for DEMASK) because its KL-based criterion identifies that positions following an EOS token should also be EOS---a trivially predictable pattern that inflates apparent efficiency gains. This effect is amplified on benchmarks like HumanEval and MBPP, where the evaluation configuration (following the original Dream benchmark) sets max generation length to 768 and 1024 tokens respectively, while the model typically produces much shorter outputs, leaving many positions as EOS tokens that KLASS can trivially unmask in bulk. With the EOS optimization applied uniformly, all methods benefit from early termination, and KLASS's runtime advantage largely disappears (compare with Table~\ref{tab:dream-klass}).

\begin{table}[h]
\centering
\caption{Results \emph{without} EOS optimization. Runtime in minutes. Best accuracy per benchmark in \textbf{bold}.}
\label{tab:eos-ablation}
\small
\setlength{\tabcolsep}{4pt}
\begin{tabular}{@{}l cc cc@{}}
\toprule
& \multicolumn{2}{c}{DEMASK} & \multicolumn{2}{c}{KLASS} \\
\cmidrule(lr){2-3} \cmidrule(l){4-5}
Benchmark & Acc (\%) & Time (min) & Acc (\%) & Time (min) \\
\midrule
MMLU-Pro & \textbf{47.6} & 89.6 & 45.7 & 126.7 \\
GSM8K & \textbf{83.7} & 10.1 & 82.6 & 14.3 \\
HumanEval & \textbf{58.5} & 5.8 & 54.3 & 4.1 \\
MBPP & 56.6 & 19.3 & \textbf{57.0} & 6.9 \\
\bottomrule
\end{tabular}
\end{table}

\section{Hyperparameter Search Grids}
\label{app:hyperparams}

We conduct grid searches on GSM8K to select hyperparameters for both DEMASK and KLASS.

\paragraph{DEMASK.} We search over dependency threshold $\tau \in \{0.5, 0.4, 0.3, 0.2, 0.1, 0.08, 0.06, 0.04, 0.02, 0.01, 0.003, 0.001\}$ and confidence threshold $\gamma \in \{0.9, 0.7, 0.5, 0.3, 0.1\}$, yielding 60 configurations. We select $\tau = 0.04$ and $\gamma = 0.9$.

\paragraph{KLASS.} We search over KL threshold $\in \{0.02, 0.015, 0.01, 0.005, 0.001, 0.0003, 0.0001, 0.00003, 0.00001\}$ and confidence threshold $\in \{0.5, 0.6, 0.7, 0.8, 0.9\}$, yielding 45 configurations. We select KL threshold $= 0.0003$ and confidence threshold $= 0.9$.

\end{document}